\documentclass[twoside,11pt]{article}
\usepackage[preprint]{jmlr2e}
\usepackage{mypaper}

\newcommand{\mytitle}{Explainable Reinforcement Learning via\\Temporal Policy Decomposition}

\jmlrheading{23}{2022}{1-\pageref{LastPage}}{1/21; Revised 5/22}{9/22}{21-0000}{%
    Franco Ruggeri, 
    Alessio Russo, 
    Rafia Inam, and 
    Karl Henrik Johansson%
}

\ShortHeadings{\mytitle}{Ruggeri, Russo, Inam, and Johansson}

\firstpageno{1}

\begin{document}

\title{\mytitle}
\author{%
    \name Franco Ruggeri
    \email franco.ruggeri@ericsson.com \\
    \addr Ericsson Research, KTH Royal Institute of Technology
    \email fruggeri@kth.se \\
    \addr Stockholm, Sweden
    \AND
    \name Alessio Russo
    \email arusso2@bu.edu \\
    \addr Boston University \\
    Boston, USA
    \AND
    \name Rafia Inam
    \email rafia.inam@ericsson.com \\
    \addr Ericsson Research, KTH Royal Institute of Technology 
    \email raina@kth.com \\ 
    \addr Stockholm, Sweden
    \AND
    \name Karl Henrik Johansson
    \email kallej@kth.se \\
    \addr KTH Royal Institute of Technology \\
    Stockholm, Sweden
}
\editor{My editor}

\maketitle

\begin{abstract}%
We investigate the explainability of \gls{rl} policies from a temporal perspective, focusing on the sequence of future outcomes associated with individual actions. In \gls{rl}, value functions compress information about rewards collected across multiple trajectories and over an infinite horizon, allowing a compact form of knowledge representation. However, this compression obscures the temporal details inherent in sequential decision-making, presenting a key challenge for interpretability. We present \gls{tpd}, a novel explainability approach that explains individual \gls{rl} actions in terms of their \gls{efo}. These explanations decompose generalized value functions into a sequence of \glspl{efo}, one for each time step up to a prediction horizon of interest, revealing insights into when specific outcomes are expected to occur. We leverage fixed-horizon temporal difference learning to devise an off-policy method for learning \glspl{efo} for both optimal and suboptimal actions, enabling contrastive explanations consisting of \glspl{efo} for different state-action pairs. Our experiments demonstrate that \gls{tpd} generates accurate explanations that 
\begin{enumerate*}[label=(\roman*)]
    \item clarify the policy's future strategy and anticipated trajectory for a given action and
    \item improve understanding of the reward composition, facilitating fine-tuning of the reward function to align with human expectations.
\end{enumerate*}
\end{abstract}

\begin{keywords}
  \glsentrylong{xrl}, \glsentrylong{xai}, \glsentrylong{rl}, \glsentrylong{ai}, Explainability
\end{keywords}

\glsresetall
\section{Introduction}

\Gls{rl} has demonstrated considerable success in a wide range of applications, including complex domains like robotics \citep{kober-2013,polydoros-2017} and telecommunications \citep{saxena-2022,vannella-2022}. Despite these achievements, a persistent challenge in \gls{rl} is its lack of interpretability. The opacity of \gls{rl} solutions hinders their wider adoption and deployment to real-world applications, especially in safety-critical or highly regulated fields. This challenge has given rise to the field of \gls{xrl}, which aims to provide user-interpretable explanations to make \gls{rl} systems more transparent and interpretable \citep{milani-2023,puiutta-2020}. A clear explanation of why a model makes certain decisions is essential for gaining trust from human users as well as ensuring robust and reliable deployment.

Research in \gls{xrl} has produced various methods to address the interpretability challenge. One line of research has focused on explaining the function approximators employed in deep \gls{rl} algorithms, where deep neural networks are employed to represent value functions or policies \citep{greydanus-2018,lundberg-2017}. These methods are generally designed for supervised learning settings, thus neglecting the sequential nature distinctive of \gls{rl}. A promising alternative direction consists of learning additional value functions, often referred to as \glspl{gvf}, to explain the dynamics of the \gls{rl} environment under a given \gls{rl} policy \citep{juozapaitis-2019,yau-2020,lin-2021}. These \glspl{gvf} are usually designed to decompose across a particular dimension the primary reward-based value function, which is used by the \gls{rl} agent for decision-making.

While these approaches offer valuable insights, they completely overlook the time dimension of \gls{rl}-based decision-making. A core problem in explaining \gls{rl} behavior lies, in fact, in how value functions obscure temporal information. Specifically, value functions summarize the expected rewards over multiple trajectories and an infinite horizon, effectively hiding \textit{when} and \textit{where} rewards are expected to be collected. The temporal ambiguity is especially problematic due to delayed rewards (i.e., rewards that are the result of a sequence of actions). Therefore, effective explanations should clarify the underlying strategy of an \gls{rl} policy: What rewards does the policy intend to collect, and over what time frame? Incorporating this temporal aspect into explanations is essential for a more comprehensive understanding of \gls{rl} behavior. 

One attempt to address this gap was made by \citet{vanderwaa-2018}, who proposed analyzing the most likely future trajectory under an \gls{rl} policy and comparing it with a counterfactual aligned with human expectations. However, this approach has two notable limitations:
\begin{enumerate*}[label=(\roman*)]
    \item it exclusively focuses on the most likely trajectory, potentially neglecting a significant portion of the possible outcomes in stochastic environments, and 
    \item it assumes a fully known environment model to compute the most likely trajectory with Monte Carlo simulations.
\end{enumerate*}
Additionally, while model-based \gls{rl} algorithms aim to learn the environment's dynamics by modeling the next-state distribution \citep{moerland-2023}, they cannot accurately predict state distributions at future time steps due to the challenges of uncertainty propagation and compound errors. As a result, there remains a need for methods that can explain actions in terms of their outcomes in future time steps while also accounting for the inherent uncertainty of \gls{rl} decisions.

In this paper, we present \gls{tpd}, a novel \gls{xrl} method that explains \gls{rl} actions in terms of their \glspl{efo}. Given an initial state-action pair, an outcome function, and an \gls{rl} policy, an \gls{efo} provides information about the outcome after a specific number of time steps. Our method generates explanations, consisting of \glspl{efo} for each future time step up to a fixed prediction horizon, that offers time-granular insights into the expected future trajectory. This approach is orthogonal to existing value decomposition techniques \citep{juozapaitis-2019,yau-2020,lin-2021}, as the set of \glspl{efo} for a range of time steps can be seen as a temporal decomposition of the \gls{gvf} associated with the outcome function. In this regard, we explore two important classes of outcomes---rewards and events---that yield meaningful temporal decompositions. To learn \glspl{efo}, we propose an off-policy learning method based on \gls{fhtd} learning \citep{deasis-2020}. The off-policy nature of \gls{tpd} supports the generation of contrastive explanations, often preferred by human users \citep{miller-2019}, by comparing optimal and non-optimal actions. Additionally, we prove the convergence of tabular \gls{fhtd} learning under less restrictive assumptions than those presented in \citet{deasis-2020}.

\paragraph{Contributions} Our main contributions are summarized as follows:
\begin{itemize}
    \item We introduce \gls{tpd}, a novel \gls{xrl} method that explains \gls{rl} actions by decomposing \glspl{gvf} across the temporal dimension. The generated explanations offer time-granular insights into the consequences of each decision.
    \item We analyze \gls{tpd} for two relevant classes of outcomes---rewards and events---providing a flexible framework for interpreting \gls{rl} behavior.
    \item We prove the convergence of \gls{fhtd} learning in tabular settings under less restrictive assumptions compared to the proof by \citet{deasis-2020}.
    \item We conduct an empirical evaluation of \gls{tpd}, demonstrating its ability to generate intuitive and reliable explanations.
\end{itemize}

\paragraph{Outline} The remainder of the paper is organized as follows. In \cref{sec:related-work}, we discuss related work. In \cref{sec:background}, we provide a brief background on \gls{rl} and \gls{fhtd} learning. In \cref{sec:temporal-policy-decomposition}, we present the \gls{tpd} method. In \cref{sec:experiments}, we describe the experiments and discuss the results. Finally, in \cref{sec:conclusion}, we conclude with a discussion and outline directions for future work.

\section{Related Work}
\label{sec:related-work}

The literature of \gls{xrl} presents methods that tackle the interpretability challenge from different perspectives, as outlined by \citet{milani-2023}. The most relevant line of research for this paper is the training of auxiliary value functions, sometimes referred to as \glspl{gvf}, to offer a more interpretable view of the primary reward-based value function used for decision-making. Based on this approach, \citet{juozapaitis-2019} explored the idea of decomposing the value function into separate components that align with the reward composition, assuming that the reward is formulated as a sum of reward components. The value composition can then be used as an explanation of how each reward component affects decision-making. \citet{lin-2021} generalized this method by introducing user-defined outcomes in place of traditional rewards. Specifically, they proposed a two-stage model where
\begin{enumerate*}[label=(\roman*)]
    \item the first stage learns the \glspl{gvf} as functions of state and action, capturing the user-defined outcomes, and
    \item the second stage then uses these \glspl{gvf} to predict the reward-based value function.
\end{enumerate*}
This architecture can be trained end-to-end and generate explanations in terms of \glspl{gvf}, which represent expected discounted sums of the user-defined outcomes. Instead, \citet{yau-2020} proposed to train, in parallel to the reward-based value function, state-action occupancy maps estimating the expected occurrence of state-action pairs in the future trajectory. Despite providing relevant insights, none of these works considered the temporal dimension, as they all rely on discounted sums spanning an infinite horizon. In this paper, while we use the \gls{gvf} concept and the idea of learning additional \glspl{gvf} to enhance the transparency of the main reward-based value function, we focus on a temporal decomposition of \glspl{gvf} to explain actions in terms of their expected outcomes in future time steps.

Another closely related study is by \citet{vanderwaa-2018}, who devised a method for generating contrastive explanations by comparing the most likely trajectory under the given \gls{rl} policy with the most likely trajectory from a different policy aligned with human expectations. This method produces explanations as a sequence of abstracted states, actions, and rewards. However, it has two main limitations:
\begin{enumerate*}[label=(\roman*)]
\item it focuses exclusively on the most likely trajectory, potentially overlooking a significant range of possible outcomes in stochastic environments, and
\item it requires a fully known environment model to identify the most likely trajectory using Monte Carlo simulations.
\end{enumerate*}
In contrast, our approach overcomes these limitations by predicting \gls{efo} based on the full distribution of trajectories and by learning these \glspl{efo} in a model-free manner.

Model-based \gls{rl}, which involves explicitly learning the environment’s dynamics to guide decision-making \citep{polydoros-2017}, is also a related research area. State-of-the-art methods in this field utilize ensembles of probabilistic neural networks to model the next-state distribution \citep{chua-2018,janner-2019}. These dynamics models can then generate trajectories through sequential rollouts for planning or for data augmentation. While such trajectories could, in principle, be used to explain actions, there are several limitations. First, a well-known challenge with such models is the compounding error, where small inaccuracies accumulate over a rollout, leading to significant errors in long-horizon trajectories. Second, these models do not estimate the trajectory distribution, as it is not possible to propagate uncertainty analytically over multiple time steps. Thus, an explanation would rely on a limited set of sampled trajectories, which would inevitably capture only a fraction of the full trajectory distribution. A method to manage uncertainty propagation through the dynamics model was introduced by \citet{deisenroth-2011}, using Gaussian processes and Gaussian approximations at each time step. While these approximations have been shown effective for policy optimization, they are problematic when trajectories are used directly for explainability, where trustworthiness and reliability are crucial. In this paper, we take a different approach by designing models that estimate the desired quantities directly at specific future time steps, avoiding relying on rollouts and dynamics models altogether.

\section{Background}
\label{sec:background}

In this section, we briefly present the technical background and notation on \gls{rl} and \gls{fhtd} learning. These concepts are extensively used in the remainder of the paper.

\subsection{\glsentrylong{rl}}

An \gls{rl} problem is mathematically modeled as a \gls{mdp} \citep{sutton-2018}. An \gls{mdp} is defined by a tuple $\mathcal{M}= \langle \mathcal{S}, \mathcal{A}, p, r, \gamma \rangle$, where $\mathcal{S}$ is the state space, $\mathcal{A}$ is the action space, $p : \mathcal{S} \times \mathcal{A} \times \mathcal{S} \to [0,1]$ are the transition dynamics, $r : \mathcal{S} \times \mathcal{A} \times \mathcal{S} \to [-1,1]$ is the reward function, which we assume bounded\footnote{We assume intervals $[-1,1]$ to simplify the analysis, which can be extended to generic bounded intervals. The same applies to outcome functions.}, and $\gamma \in [0,1]$ is the discount factor. Specifically, at time $t \in \mathbb{Z}_{\ge0}$, an agent observes the environment state $s_t \in \mathcal{S}$ and selects an action $a \in \mathcal{A}$. Then, the \gls{mdp} transitions to a new state $s_{t+1} \sim p(\cdot | s_t,a_t)$ and gives the agent a reward $r_t = r(s_t,a_t,s_{t+1})$. An \gls{mdp} is called finite if both its state and action spaces are discrete and continuous if either the state or action space is continuous. A policy $\pi: \mathcal{S} \to \Delta(\mathcal{A})$ defines the probability $\pi(a|s)$ of taking action $a \in \mathcal{A}$ in state $s \in \mathcal{S}$. The state-action value function under a policy $\pi$, denoted $Q^{\pi}$, represents the expected return collected by the policy $\pi$ conditioned to an initial state-action pair:
\begin{align}
    \label{eq:q-function}
    Q^\pi(s,a) &= \mathbb{E}\left[\sum_{t=0}^{\infty} \gamma^t r_t \middle| s_0=s,a_0=a, \pi \right]
\end{align}
where the condition on $\pi$ indicates $a_t \sim \pi(\cdot|s_t)$ for $t>0$.
Instead, the state value function under a policy $\pi$, denoted $V^{\pi}$, is conditioned only on the initial state: $V^{\pi}(s) = \mathbb{E}_{a \sim \pi(\cdot|s)} [Q^{\pi}(s,a)]$. A generalization of the value function concept is the \gls{gvf}, which replaces rewards with generic outcomes \citep{sutton-2011,lin-2021}. Specifically, given a measurable and bounded outcome function $o: \mathcal{S} \times \mathcal{A} \times \mathcal{S} \to [-1,1]$, with $o_t = o(s_t,a_t,s_{t+1})$, the state-action \gls{gvf} is:
\begin{align}
    Q^{\pi}_o(s,a) = \mathbb{E}\left[\sum_{t=0}^{\infty} \gamma^t o_t \middle| s_0=s, a_0=a, \pi \right]
\end{align}
where $a_t \sim \pi(\cdot|s_t)$ for $t>0$. The two fundamental \gls{rl} problems are \textit{prediction} and \textit{control}. Prediction consists of learning the value function under a given policy $\pi$, $V^{\pi}$ or $Q^{\pi}$, whereas control involves learning an optimal policy $\pi^*$ that maximizes the value in each state. An \gls{rl} algorithm learns to perform prediction or control based on experiences obtained through interactions between the agent and the environment without prior knowledge of the transition probabilities.

\subsection{\glsentrylong{fhtd} Learning}

The \textit{fixed-horizon} state-action value function under a policy $\pi$, denoted $Q_h^{\pi}$, is defined as the expected return collected over a fixed horizon: 
\begin{align}
    Q^{\pi}_h(s,a) = \mathbb{E}\left[\sum_{t=0}^{h-1} \gamma^t r_t \middle| s_0=s,a_0=a, \pi \right]
\end{align}
where $a_t \sim \pi(\cdot|s_t)$ for $t=1,\dots,h-1$. A state-of-the-art \gls{rl} prediction algorithm that learns this function is \gls{fhtd} learning \citep{deasis-2020}. Consider a transition at time $t$ from state $s_t \in \mathcal{S}$ to state $s_{t+1} \in \mathcal{S}$ after taking action $a_t \in \mathcal{A}$, with reward $r_t = r(s_t, a_t, s_{t+1})$. In a finite \gls{mdp}, let $\hat{Q}^{(t)}_h \in \mathbb{R}^{|\mathcal{S}| \times |\mathcal{A}|}$ be the estimate of $Q^{\pi}_h$ at time $t$. The \gls{fhtd} update is:
\begin{align}
    \hat{G}_h &= r_t + \gamma \hat{Q}_{h-1}(s_{t+1},a_{t+1}), \; a_{t+1} \sim \pi(\cdot|s_{t+1}) \\
    \label{eq:fhtd-learning}
    \hat{Q}^{(t+1)}_h(s_t,a_t) &\gets \hat{Q}^{(t)}_h(s_t,a_t) + \alpha_{t,h}(s_t,a_t) [\hat{G}_h - \hat{Q}^{(t)}_h(s_t,a_t)]
\end{align}
for $h=0, \dots, H-1$, where $H \in \mathbb{N}$ is the prediction horizon and $\alpha_{t,h}$ a learning rate schedule. 

\section{Temporal Policy Decomposition}
\label{sec:temporal-policy-decomposition}

In this section, we first define the concept of \gls{efo} and clarify its utility for explainability purposes. Then, we devise an \gls{xrl} method to explain actions selected by a given black-box policy, such as one learned with a \gls{rl} control algorithm, in terms of \glspl{efo}.

\subsection{\glsentrylongpl{efo}}

Since \gls{rl} involves sequential decision-making, each \gls{rl} action should be explained in terms of the sequence of expected outcomes happening after the action is taken. This type of explanation allows us to quantify and interpret the potential effects of actions across different time steps, thus enabling more explainable and interpretable decision-making. To formalize this idea, we introduce the following definition:
\begin{definition}[\glsentrylong{efo}]
    Given a policy $\pi: \mathcal{S} \to \Delta(\mathcal{A})$, a measurable and bounded function $o: \mathcal{S} \times \mathcal{A} \times \mathcal{S} \to [-1,1]$, an initial state $s \in \mathcal{S}$ and a starting action $a \in \mathcal{A}$, the \gls{efo} after $h$ time steps under the policy $\pi$ is:
    \begin{align}
        O_h^{\pi}(s,a) = \mathbb{E} \left[ o(s_h,a_h,s_{h+1}) \middle| s_0=s,a_0=a,\pi \right]
    \end{align}
    where $a_i \sim \pi(\cdot|s_i)$ for $i=1,\dots,h-1$.
\end{definition}
A state-action \gls{gvf} can be expressed as:
\begin{align}
    Q^{\pi}_o(s,a) = \sum_{h=0}^{\infty} \gamma^h O^{\pi}_h(s,a)
\end{align}
Intuitively, the \glspl{efo} allow for a temporal decomposition of \glspl{gvf} under the policy. In the case of a state-action value function $Q^{\pi}$, this approach effectively restores temporal information about \textit{when} rewards are expected to be collected that is aggregated and lost in \gls{rl} algorithms. For this reason, we call our method \glsentryfull{tpd}.

\begin{figure}
    \centering
    \includegraphics[trim={15 15 15 10},clip,width=\textwidth]{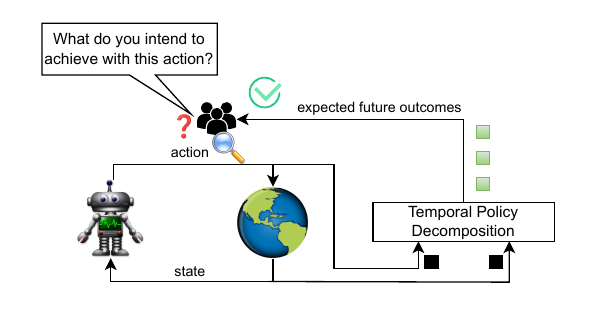}
    \caption{Conceptual overview of the method. Given the sequential interaction between agent and environment, human observers need to interpret each selected action in terms of its expected impact on the future trajectory. Thus, \gls{tpd} explains the action by generating the \glspl{efo} for a long-enough prediction horizon.}
    \label{fig:high-level-method}
\end{figure}

We propose to generate explanations that include \glspl{efo} up to a finite prediction horizon, as dealing with an infinite number of \glspl{efo} is infeasible in practice. Combining the concepts of \gls{gvf} \citep{sutton-2011,lin-2021} and fixed-horizon value function \citep{deasis-2020}, we have the following definition:
\begin{definition}[\glsentrylong{fhgvf}]
    Given a policy $\pi: \mathcal{S} \to \Delta(\mathcal{A})$, a prediction horizon $H \in \mathbb{N}$, and a measurable and bounded function $o: \mathcal{S} \times \mathcal{A} \times \mathcal{S} \to [-1,1]$, the \gls{fhgvf} under the policy $\pi$ is:
    \begin{align}
        Q_{o,H}^{\pi}(s,a) = \mathbb{E} \left[ \sum_{h=0}^{H-1} \gamma^h o_h \middle| s_0=s, a_0=a, \pi \right]
    \end{align}
    where $o_h = o(s_h,a_h,o_{h+1})$ and $a_h \sim \pi(\cdot|s_h)$ for $h = 1, \dots, H-1$.
\end{definition}
The temporal decomposition with a finite prediction horizon $H \in \mathbb{N}$ can thus be expressed as:
\begin{align}
    Q^{\pi}_o(s,a) = Q^{\pi}_{o,H}(s,a) + c = \sum_{h=0}^{H-1} \gamma^h O_h^{\pi}(s,a) + c
\end{align}
where $|c| \le \frac{\gamma^H}{1-\gamma}$ represents the value part neglected by the explanation. With a sufficiently large prediction horizon, this approach can reveal key insights into the expected future trajectory under the policy $\pi$, starting in state $s \in \mathcal{S}$ and taking action $a \in \mathcal{A}$. A conceptual overview of the method is provided in \cref{fig:high-level-method}.

The concept of \gls{efo} provides a structured approach to capturing relevant information at a desired level of abstraction, avoiding the complexity of low-level features typically used in deep \gls{rl}. By defining outcome functions that represent quantities of interest, users can focus on the essential, interpretable aspects of the \gls{mdp} and gain insights into the Markov chain induced by steering the \gls{mdp} with the policy $\pi$. Specifically, given a set of interesting outcome functions $[o_k]_{k=1}^K$, for a state $s \in \mathcal{S}$ and an action $a \in \mathcal{A}$, the explanation consists of the sequence of \glspl{efo} $[O^{\pi}_{h,k}(s,a)]_{h=0,k=1}^{H-1,K}$. 

Furthermore, \glspl{efo} can also provide \textit{contrastive} explanations by comparing the \glspl{efo} of different state-action pairs. For example, given the selected action $a_1 \in \mathcal{A}$ (fact) and a different action $a_2 \in \mathcal{A}$ (foil), the comparison of the respective \glspl{efo} $[O^{\pi}_h(s,a_i)]_{h=0}^H$, $i \in \{1,2\}$ (e.g., through the difference) answers the contrastive question, “Why did you choose action $a_1$ over action $a_2$?”. Such contrastive explanations are often more intuitive and valuable for human users \citep{miller-2019}. Notably, the ability to generate contrastive explanations is the key motivation for defining \glspl{efo} as functions of both state and action rather than only state.

\subsection{Relevant Classes of \glsentrylongpl{efo}}

In the following subsection, we discuss two particularly interesting types of outcomes, namely rewards and events.

\subsubsection{Expected Future Rewards}

Rewards are natural choices for outcomes, as \gls{rl} control algorithms are designed to maximize expected discounted cumulative rewards. Thus, a simple and effective outcome function is the reward function itself, $r(s,a,s^{\prime})$, which allows for a temporal decomposition of the state-action value function $Q^{\pi}$. This decomposition offers insights into when rewards are collected under a given policy $\pi$.

Inspired by \citet{juozapaitis-2019}, when rewards are composed of multiple components, each representing a distinct aspect, we can gain an even deeper understanding. Let $r(s,a,s^{\prime}) = g(\mathbf{r}(s,a,s^{\prime}))$ be the reward function, with $g: \mathbb{R}^K \to \mathbb{R}$ linear and $\mathbf{r}(s,a,s^{\prime}) = [r_k(s,a,s^{\prime})]_{k=1}^K$ a vector of reward components. Let $R^{\pi}_{h,k}(s,a)$ be the \gls{efo} corresponding to $r_k$, for $k=1, \dots, K$. The state-action value function in \cref{eq:q-function} can then be written as:
\begin{align}
Q^{\pi}(s,a) 
&= \sum_{h=0}^{\infty} \mathbb{E}\left[\gamma^h g(\mathbf{r}(s_h,a_h,s_{h+1})) \middle| s_0=s,a_0=a,\pi \right] \nonumber \\
&= g\left( \sum_{h=0}^{\infty} \gamma^h \mathbb{E}\left[ \mathbf{r}(s_h,a_h,s_{h+1}) \middle| s_0=s,a_0=a,\pi \right] \right) \nonumber \\
&= g\left( \sum_{h=0}^{\infty} \gamma^h [R^{\pi}_{h,k}(s,a)]_{k=1}^K \right)
\end{align}
where we have used the linearity of $g$ to swap it with expectation and sum. This formulation extends the reward decomposition method proposed by \citet{juozapaitis-2019}. In addition to decomposing the value into distinct value components mirroring the reward structure, this approach also introduces a temporal aspect, allowing for a more granular understanding of how reward components are accumulated over time under a given policy.

\subsubsection{Probabilities of Future Events}
\label{sec:probabilities-of-future-events}
A particularly interesting type of outcome is an event. Let $\mathcal{X} = \mathcal{S} \times \mathcal{A} \times \mathcal{S}$ be the set of all possible transitions $(s,a,s^{\prime})$ and let $\mathcal{E} \subseteq \mathcal{X}$ be the set of transitions where the event occurs. We can then define the event as an outcome function $e$ using an indicator function:
\begin{align}
    e(s,a,s^{\prime}) = \mathbf{1}_{\mathcal{E}}(s,a,s^{\prime})
\end{align}
and the corresponding \gls{efo} $E^{\pi}_h$ becomes:
\begin{align}
    E^{\pi}_h(s,a) &= \mathbb{P}\left[(s_h,a_h,s_{h+1})\in \mathcal{E} \middle| s_0=s, a_0=a, \pi\right]
\end{align}
Intuitively, when the outcome is an event $e(s,a,s^{\prime})$, the corresponding \gls{efo} $E^{\pi}_h(s,a)$ is the probability that in $h$ time steps the event will occur when starting from state $s \in \mathcal{S}$ and taking action $a \in \mathcal{A}$. We argue that events and probabilities of future events are particularly interpretable for human users.

Events can also be used to decompose the state-action value function $Q^{\pi}$ not only over the time dimension but also over the state-action space, allowing a variant of the reward decomposition discussed in the previous section. Let $x = (s,a,s^{\prime}) \in \mathcal{X}$ be a transition. Consider a complete set of events $\{ e_1,\dots,e_K \}$ (i.e., exhaustive and mutually exclusive), with corresponding subsets $\{ \mathcal{E}_k : \mathcal{E}_k \subseteq \mathcal{X}, k = 1, \dots, K \}$, and such that each event is associated to a deterministic reward:
\begin{gather}
    \label{eq:exhaustive-event-set-assumption}
    \bigcup_{k=1}^K \mathcal{E}_k = \mathcal{X} \\
    \label{eq:mutually-exclusive-events-assumption}
    \mathcal{E}_i \cap \mathcal{E}_j = \emptyset, \; \forall i \neq j, 1 \le i,j \le K \\
    \label{eq:event-reward-map-assumption}
    r(x_i) = r(x_j), \; \forall x_i,x_j \in \mathcal{E}_k, \; k=1,\dots,K
\end{gather}
Let $R^{\pi}_h$ and $E^{\pi}_{h,k}$ be the \glspl{efo} corresponding to the reward $r$ and the events $e_k$, for $k=1,\dots,K$, respectively. The following relationship holds:
\begin{align}
    R^{\pi}_h(s,a) = \sum_{k=1}^K r_k E_{h,k}^\pi(s,a)
\end{align}
and the state-action value function in \cref{eq:q-function} can be re-written as:
\begin{align}
    Q^\pi(s,a)
    &= \sum_{h=0}^{\infty} \gamma^h R^{\pi}_h(s,a) \nonumber \\
    &= \sum_{k=1}^K r_k \sum_{h=0}^{\infty} \gamma^h E^{\pi}_{h,k}(s,a) \nonumber \\
    &= \sum_{k=1}^K r_k \mu_k^\pi(s)
\end{align}
where $\mu_k^\pi(s) = \sum_{t=0}^\infty \gamma^{t} E_{t,k}^\pi(s,a)$ is the unnormalized discounted occupancy over the event $\mathcal{E}_k$ under $\pi$. An advantage of this value decomposition using a complete set of events is that, once the probability of future events for the events are learned, it is possible to reconstruct also the expected future rewards as well as their composition across the state-action space. 

Note that, by defining reward components as $r_k(s,a,s^{\prime}) = r(s,a,s^{\prime}) \mathbf{1}_{\mathcal{E}_k}(s,a,s^{\prime})$ and the reward as $r(s,a,s^{\prime}) = \sum_{k=1}^K r_k(s,a,s^{\prime})$, we find that the reward decomposition across the state-action space is a special case of the previous section. Nevertheless, since we suggest that probabilities of future events are particularly easy to interpret for human users, learning \glspl{efo} from events and reconstructing expected future rewards from them is valuable.

\subsection{Learning \glsentrylongpl{efo}}
\label{sec:learning-efos}

Generating contrastive explanations requires learning to predict \glspl{efo} $[O^{\pi}_h(s,a)]_{h=0}^{H-1}$ for all actions $a \in \mathcal{A}$, not just for the actions selected by the target policy $\pi$, represented as $[O^{\pi}_h(s,a)]_{h=0}^{H-1}$. This requirement raises the need for an off-policy algorithm capable of utilizing data collected under an explorative behavioral policy different from the target policy that needs to be explained. 

We propose a method that meets the off-policy requirement. For a given outcome $o$, a policy $\pi$, and a prediction horizon $H$, we can write the following recursive relationship:
\begin{align}
    Q_{o,h}^{\pi}(s,a) &= Q_{o,h-1}^{\pi}(s,a) + \gamma^h O^{\pi}_h(s,a)
\end{align}
with the initialization $Q_{o,-1}^{\pi}(s,a) = 0$ for all state-action pairs. This formulation allows the problem of learning the \glspl{efo} $[O_h^{\pi}(s,a)]_{h=0}^{H-1}$ to be reframed as the task of learning \glspl{fhgvf} $[Q_{o,h}(s,a)]_{h=0}^{H-1}$. Once the \glspl{fhgvf} are learned, we can solve the following set of linear equations to find the \glspl{efo}:
\begin{align}
    \label{eq:fhgvf-decomposition}
    \begin{bmatrix}
    1 & 0 & 0 & \cdots & 0 \\
    1 & \gamma & 0 & \cdots & 0 \\
    1 & \gamma & \gamma^2 & \cdots & 0 \\
    \vdots & \vdots & \vdots & \ddots & \vdots \\
    1 & \gamma & \gamma^2 & \cdots & \gamma^{H-1}
    \end{bmatrix}
    \begin{bmatrix}
        O_0(s,a) \\
        O_1(s,a) \\
        O_2(s,a) \\
        \vdots \\
        O_{H-1}(s,a)
    \end{bmatrix}
    =
    \begin{bmatrix}
        Q_{o,0}^{\pi}(s,a) \\
        Q_{o,1}^{\pi}(s,a) \\
        Q_{o,2}^{\pi}(s,a) \\
        \vdots \\
        Q_{o,H-1}^{\pi}(s,a)
    \end{bmatrix}
\end{align}
where $\gamma \in (0,1]$ is the discount factor used while learning $Q^{\pi}_{o,h}$. We refer to solving this set of linear equations as \gls{fhgvf} decomposition.

For learning \glspl{fhgvf}, we utilize \gls{fhtd} learning, devised by \citet{deasis-2020} and summarized in \cref{sec:background}, under the assumption that the $Q$-value is linearly realizable. Although \gls{fhtd} methods were proposed as a solution to the deadly triad problem \citep{sutton-2018}, we leverage them for explainability purposes as a step of the \gls{tpd} method. Additionally, the convergence result from \citep{deasis-2020} depends on an assumption on the discount factor $\gamma$ and the features of the true $Q$-value (see \citet[Eq. 20]{deasis-2020}) that may be hard to verify in practice. In the following theorem, we overcome this limitation and provide a simple version of \gls{fhtd} learning for tabular \glspl{mdp} that does not require any assumption on $\gamma$, nor requires multiple time-scales for learning the different $Q$-values $Q_1^\pi, Q_2^\pi, \dots, Q_H^\pi$.

\begin{theorem}
    \label{thm:convergence}
    Assume that
    \begin{enumerate*}[label=(\roman*)]
     \item the behavioral policies  $(\beta_t)_t, \beta_t:{\cal S}\to \Delta({\cal A})$, with $a_t \sim \beta_t(\cdot|s_t)$, ensure that every state-action pair $(s,a)$ is visited infinitely often and
     \item for every $h$ the  sequences $\alpha_{t,h}(s,a)$ are positive, non-increasing in $t$, satisfying $\sum_{n=1}^\infty \alpha_{n,h}(s,a)=\infty$, $\sum_{n=1}^\infty \alpha_{n,h}(s,a)^2<\infty$.
    \end{enumerate*}
    Then, under these assumptions and using the update in \cref{eq:fhtd-learning}, we have that $\lim_{t\to\infty} \hat{Q}_h^{(t+1)}(s,a)= Q_h^\pi(s,a)$ almost surely for every $(s,a)$.
\end{theorem}
\begin{proof}
    See \cref{apx:theorem}.
\end{proof}

\begin{figure}[t]
    \centering
    \begin{subfigure}{0.4\textwidth}
        \centering
        \includegraphics[width=\textwidth]{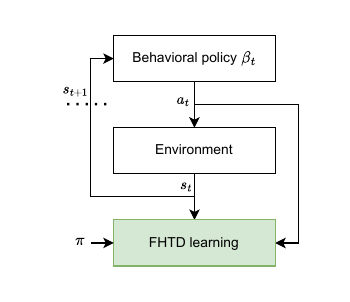}
        \caption{Training}
        \label{fig:block-diagram-training}
    \end{subfigure}
    \begin{subfigure}{0.4\textwidth}
        \centering
        \raisebox{-5mm}[\height][\depth]{%
            \hspace{0.1\textwidth}%
            \includegraphics[width=\textwidth]{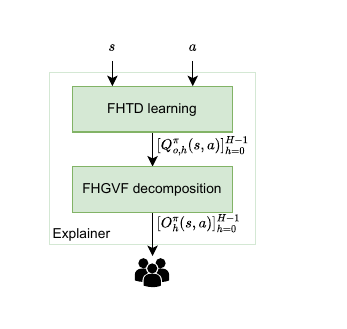}%
        }
        \caption{Inference}
        \label{fig:block-diagram-inference}
    \end{subfigure}
    \caption{Schematic representation of the training \subref{fig:block-diagram-training} and inference \subref{fig:block-diagram-inference} phases in \gls{tpd}. During training, the \glspl{fhgvf} $Q^{\pi}_{o,h}$ for $h=0,1,\dots,H-1$ using \gls{fhtd} learning \cite{}. During inference, given a state $s \in \mathcal{S}$ and an action $a \in \mathcal{A}$ to explain, the \gls{fhgvf} values $[Q^{\pi}_{o,h}(s,a)]_{h=0}^{H-1}$ are predicted and then decomposed into \glspl{efo} $[O^{\pi}_h(s,a)]_{h=0}^{H-1}$, which are presented to human users.}
    \label{fig:block-diagram}
\end{figure}

A schematic illustration of the method is provided in \cref{fig:block-diagram}. During training, \gls{fhtd} learning is employed to update estimates of the \glspl{fhgvf} $[Q_{o,h}^{\pi}(s,a)]_{h=0}^{H-1}$ for all state-action pairs, using data gathered by the behavioral policies $(\beta_t)t$, which may vary at each time step $t$. In the inference phase, the method generates an explanation for a state-action pair $(s,a)$ by computing the \glspl{efo} $[O_h^{\pi}(s,a)]_{h=0}^{H-1}$. The inference process involves first predicting the \glspl{fhgvf} $[Q_{o,h}^{\pi}(s,a)]_{h=0}^{H-1}$ and then solving \cref{eq:fhgvf-decomposition}. We refer to the sequence of these two steps as the \textit{explainer}. In an online learning setting, which this method supports, training and inference are interleaved, allowing the \gls{fhgvf} estimates to be updated continuously after each inference step.

It is worth highlighting that directly learning \glspl{efo}, instead of learning the \glspl{fhgvf}, cannot be done in an off-policy setting. A supervised learning approach or a direct stochastic approximation algorithm would be on-policy, requiring trajectories $(s_t,a_t,\dots,s_{t+h+1})$ with $a_i \sim \pi(\cdot|s_i)$ for $i = t+1, \dots, t+h$ (i.e., collected under the policy $\pi$ itself). Relying on on-policy data is problematic, as the policy $\pi$ might not be exploratory enough, leading to biased or inaccurate estimates. This issue is common in practice, as off-policy \gls{rl} control algorithms often learn a greedy policy (e.g., Q-learning). These considerations motivate our off-policy algorithm learning \glspl{fhgvf} instead of directly estimating \glspl{efo}.

\section{Experiments and Results}
\label{sec:experiments}

In this section, we describe the simulation-based experiments conducted to evaluate our method. We begin with an overview of the experimental setup, detailing the \gls{rl} environment and the training of both the policy and the explainers. After that, we present qualitative examples of explanations and provide a quantitative assessment of prediction errors compared to ground truth values.

\subsection{Experimental Setup}

We adapted the Taxi environment from Gymnasium \citep{towers-2024} by incorporating fuel consumption and traffic dynamics. Visual examples of the modified environment states are shown in \cref{fig:explanation-1-pair-1-env,fig:explanation-1-pair-2-env,fig:explanation-2-env}. The environment is formulated as the following \gls{mdp}:
\begin{itemize}
    \item \textit{State Space}: Each state is represented as an integer that encodes the taxi's position (row and column), its fuel level, and whether the passenger is in the taxi or not. To limit the state space dimension, the destination, gas station, and initial passenger location are fixed to the same corner in all episodes.
    \item \textit{Action Space}: The taxi can perform the following actions: move in four directions (south, north, east, west), pick up the passenger, drop off the passenger, and refuel at the gas station.
    \item \textit{Reward Function}: The taxi earns a reward of $+20$ for successfully dropping off the passenger at the destination and $+10$ for picking up the passenger. A penalty of $-100$ is incurred for running out of fuel or performing invalid actions, while a small penalty of $-1$ is applied for each movement action.
    \item \textit{Transition Dynamics}: Each movement action decreases the fuel level by 1 while refueling increases it by 2. The fuel level ranges from $0$ to $20$. Movement actions have a $0.1$ probability of failing, which leaves the taxi's position unchanged but still decreases the fuel level, simulating the taxi stuck in traffic. Invalid actions, such as refueling when the taxi is not at the gas station, leave the state unchanged. The episode ends with either success (when the passenger is dropped off at the destination) or failure (when the taxi runs out of fuel).
\end{itemize}
The addition of fuel consumption and traffic dynamics introduces stochasticity to the environment, making the interpretation of the taxi's behavior under a policy more challenging. Even with an optimal policy, uncertainty due to traffic conditions makes it difficult for a human observer to anticipate the policy's behavior. This unpredictability is further amplified when the policy is suboptimal, underscoring the need for explanations.

In this environment, we trained a near-optimal policy using Q-learning with action masking \citep{watkins-1992,sutton-2018}. Based on this policy, we then trained an explainer for each of the following events: \textit{dropoff}, \textit{pickup}, \textit{refuel}, \textit{failure}, \textit{traffic}, and \textit{move}. These explainers are designed to predict the probabilities of future events, as described in \cref{sec:probabilities-of-future-events}. The probability of the \textit{terminated} event is derived by exclusion without requiring an additional explainer. Since this set of events is exhaustive, and each event corresponds to a deterministic reward, the assumptions in \crefrange{eq:exhaustive-event-set-assumption}{eq:event-reward-map-assumption} are satisfied. Consequently, the expected future reward composition can be reconstructed from the predicted event probabilities, as outlined in \cref{sec:probabilities-of-future-events}. Detailed information on the training procedure can be found in \cref{apx:training-details}.

\subsection{Explainability Results}

\begin{figure}[t]
    \centering
    \begin{subfigure}{0.48\textwidth}
        \centering
        \includegraphics[width=0.9\textwidth]{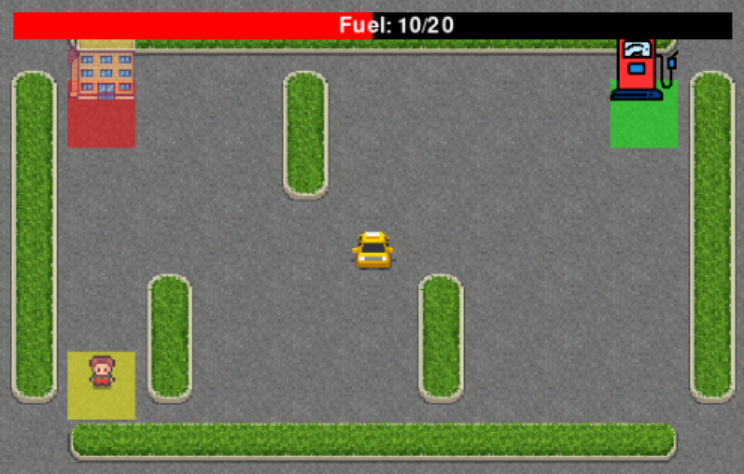}
        \caption{Environment state (fuel 10)}
        \label{fig:explanation-1-pair-1-env}
    \end{subfigure}
    \hfill
    \begin{subfigure}{0.48\textwidth}
        \centering
        \includegraphics[width=0.9\textwidth]{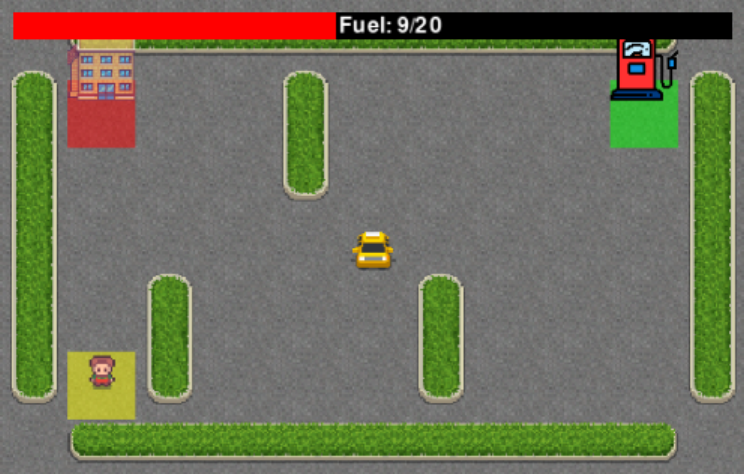}
        \caption{Environment state (fuel 9)}
        \label{fig:explanation-1-pair-2-env}
    \end{subfigure}

    \vspace{0.1cm}
    
    \begin{subfigure}{0.48\textwidth}
        \centering
        \includegraphics[trim={15 15 15 10},clip,width=0.9\textwidth]{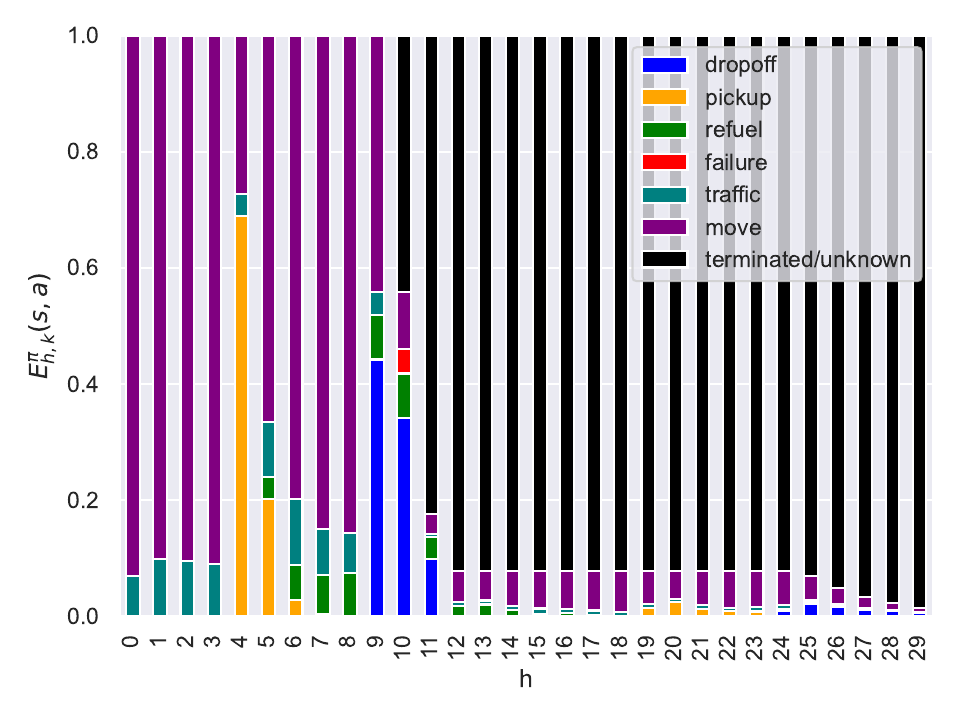}
        \caption{Probabilities of future events (fuel 10)}
        \label{fig:explanation-1-pair-1-event}
    \end{subfigure}
    \hfill
    \begin{subfigure}{0.48\textwidth}
        \centering
        \includegraphics[trim={15 15 15 10},clip,width=0.9\textwidth]{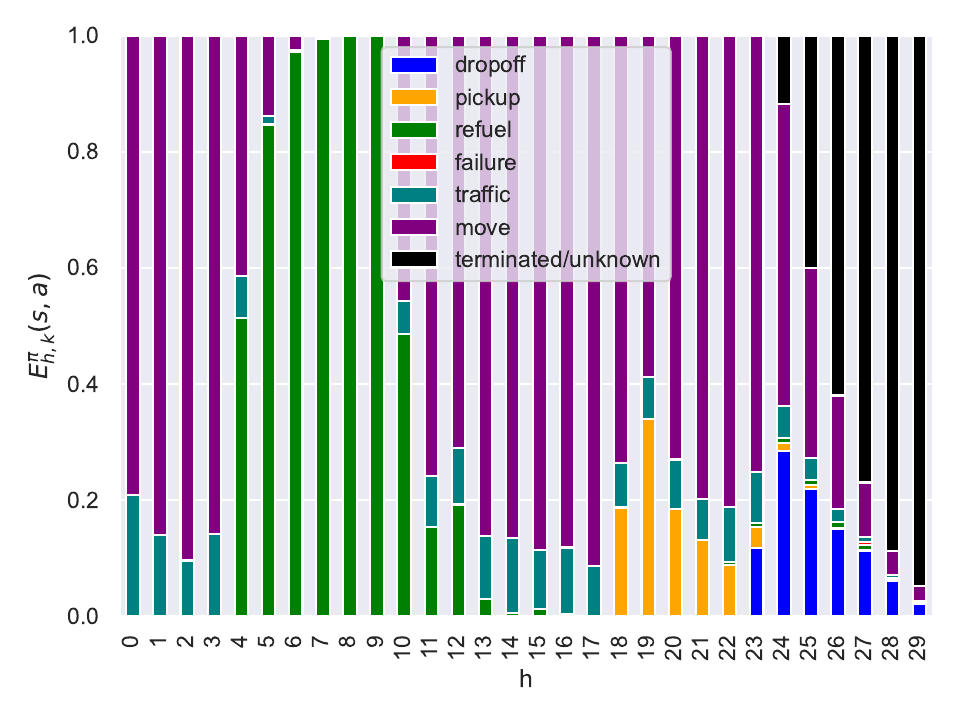}
        \caption{Probabilities of future events (fuel 9)}
        \label{fig:explanation-1-pair-2-event}
    \end{subfigure}

    \vspace{0.1cm}
    
    \begin{subfigure}{0.48\textwidth}
        \centering
        \includegraphics[trim={15 15 15 10},clip,width=0.9\textwidth]{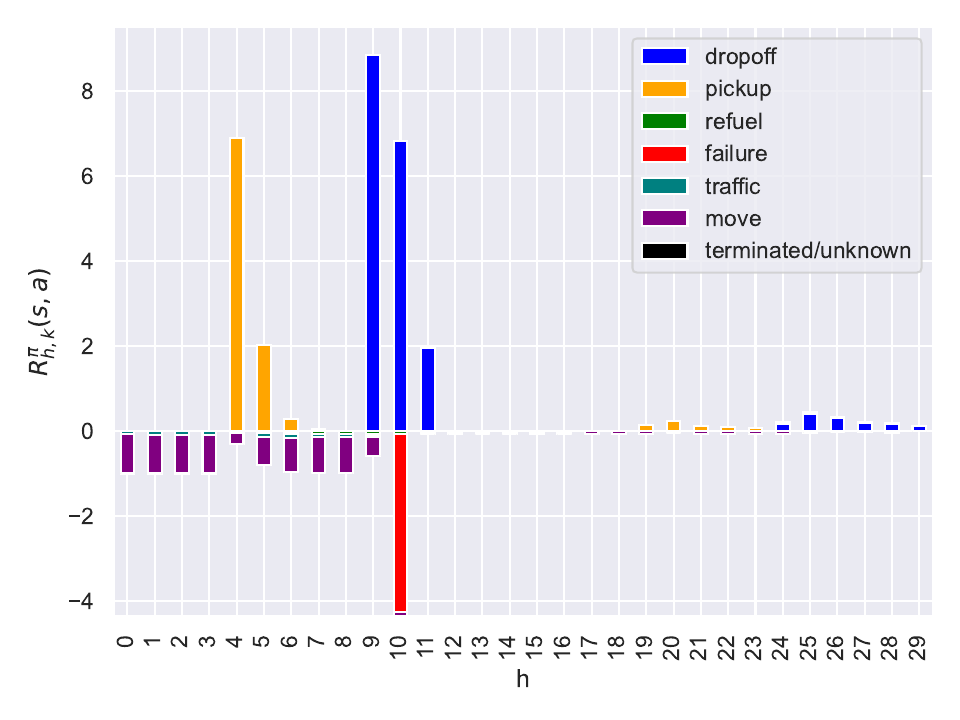}
        \caption{Expected reward components (fuel 10)}
        \label{fig:explanation-1-pair-1-reward-component}
    \end{subfigure}
    \hfill
    \begin{subfigure}{0.48\textwidth}
        \centering
        \includegraphics[trim={15 15 15 10},clip,width=0.9\textwidth]{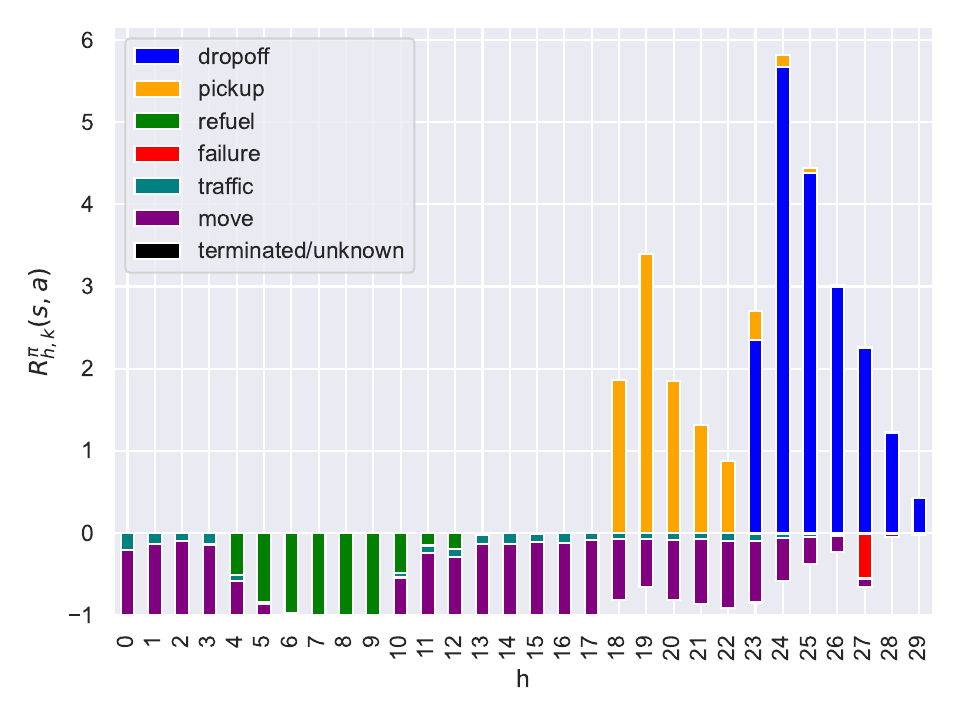}
        \caption{Expected reward components (fuel 9)}
        \label{fig:explanation-1-pair-2-reward-component}
    \end{subfigure}

    \caption{Explanations for two environment states with different initial fuel levels: $10$ \subref{fig:explanation-1-pair-1-env} and $9$ \subref{fig:explanation-1-pair-2-env}. In both scenarios, the action analyzed is the one chosen by the policy. The probabilities of future events \subref{fig:explanation-1-pair-1-event}-\subref{fig:explanation-1-pair-2-event} reveal the policy’s strategy, while the expected future rewards \subref{fig:explanation-1-pair-1-reward-component}-\subref{fig:explanation-1-pair-2-reward-component} demonstrate how the probabilities of events map to the expected rewards.}
    \label{fig:explanation-1}
\end{figure}

\Cref{fig:explanation-1} illustrates explanations for two environment states that differ only in their fuel levels (\cref{fig:explanation-1-pair-1-env} and \cref{fig:explanation-1-pair-2-env}). In both cases, the action under analysis is the one selected by the policy. The probabilities of future events (\cref{fig:explanation-1-pair-1-event} and \cref{fig:explanation-1-pair-2-event}) reveal the main strategy of the policy in both states. With a fuel level of $10$ (\cref{fig:explanation-1-pair-1-event}), the policy aims to complete the task without refueling, as shown by the high probability of picking up the passenger in steps $4$-$5$ and dropping them off in steps $9$-$10$. In contrast, when the initial fuel level is $9$ (\cref{fig:explanation-1-pair-2-event}), the policy prioritizes refueling, indicated by a high probability of refueling within steps $4$-$12$, followed by pickup and dropoff events starting from step $18$. This comparison highlights the fuel threshold the policy has learned, determining whether refueling is necessary or if the task can be completed directly. Additionally, the predictions effectively capture the probability of traffic---a factor unknown to a human observer. This traffic uncertainty explains why refueling remains a possible outcome even with a starting fuel level of $10$. If traffic occurs early in the episode, the policy adapts by heading to the gas station. 

Further insights come from the expected future rewards (\cref{fig:explanation-1-pair-1-reward-component} and \cref{fig:explanation-1-pair-2-reward-component}), which highlight the impact of event probabilities and strategy choices on rewards. With an initial fuel level of $10$, the positive rewards for pickup and dropoff are expected to occur sooner (around steps $4$-$5$ for pickup and $9$-$11$ for dropoff) compared to the $9$ fuel level scenario (around steps $18$-$23$ for pickup and $23$-$29$ for dropoff). On the other hand, the considerations for negative rewards are opposite opposite. In the first case, a potential negative reward at step $10$ reflects a possible failure due to running out of fuel, assuming three traffic events. With a starting fuel level of $9$, the episode involves more frequent but smaller penalties due to movement or refueling (steps $0$-$23$), with a reduced chance of failure at step $27$, as the policy carefully ensures sufficient fuel before leaving the gas station. These insights could be used for fine-tuning the failure penalty and encourage refueling even when starting with a fuel level of $10$. Thus, besides enhancing transparency, the method is also useful for aiding the reward engineering process.

\begin{figure}[t]
    \centering
    \begin{subfigure}{0.48\textwidth}
        \centering
        \includegraphics[width=0.9\textwidth]{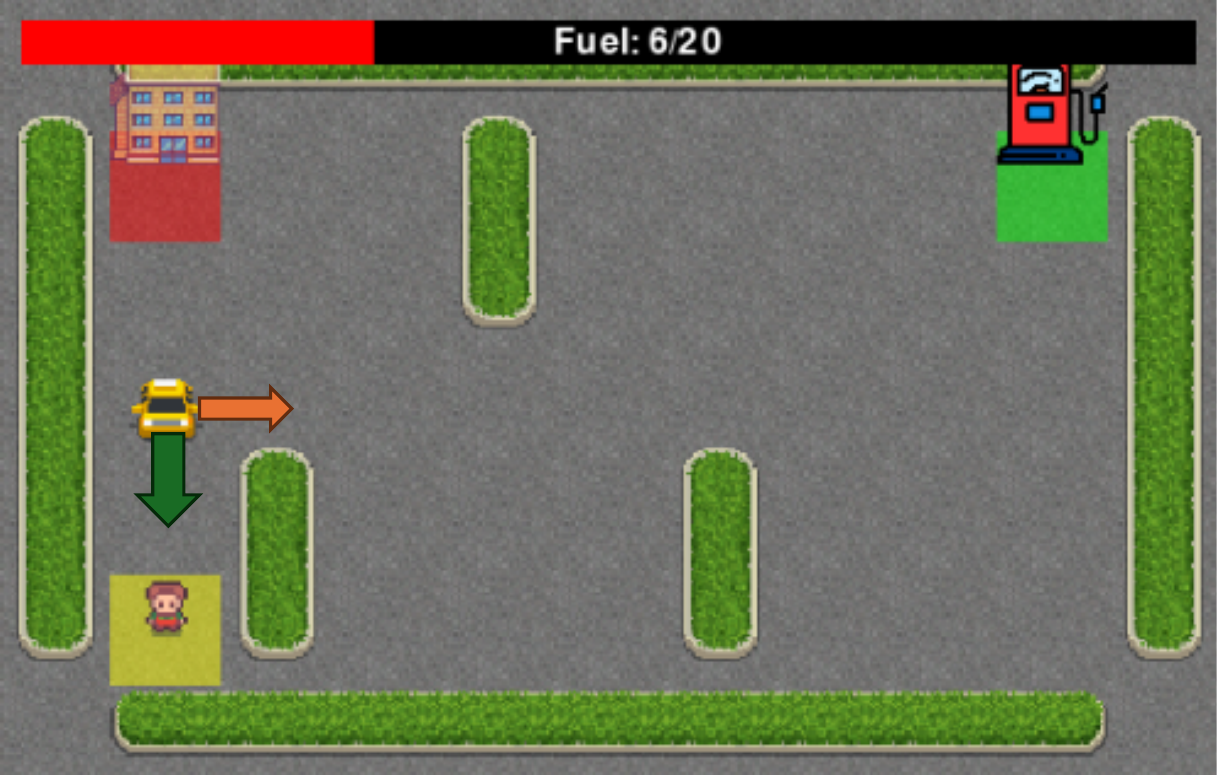}
        \caption{Environment state}
        \label{fig:explanation-2-env}
    \end{subfigure}
    \hfill
    \begin{subfigure}{0.48\textwidth}
        \centering
        \includegraphics[trim={15 15 15 10},clip,width=0.9\textwidth]{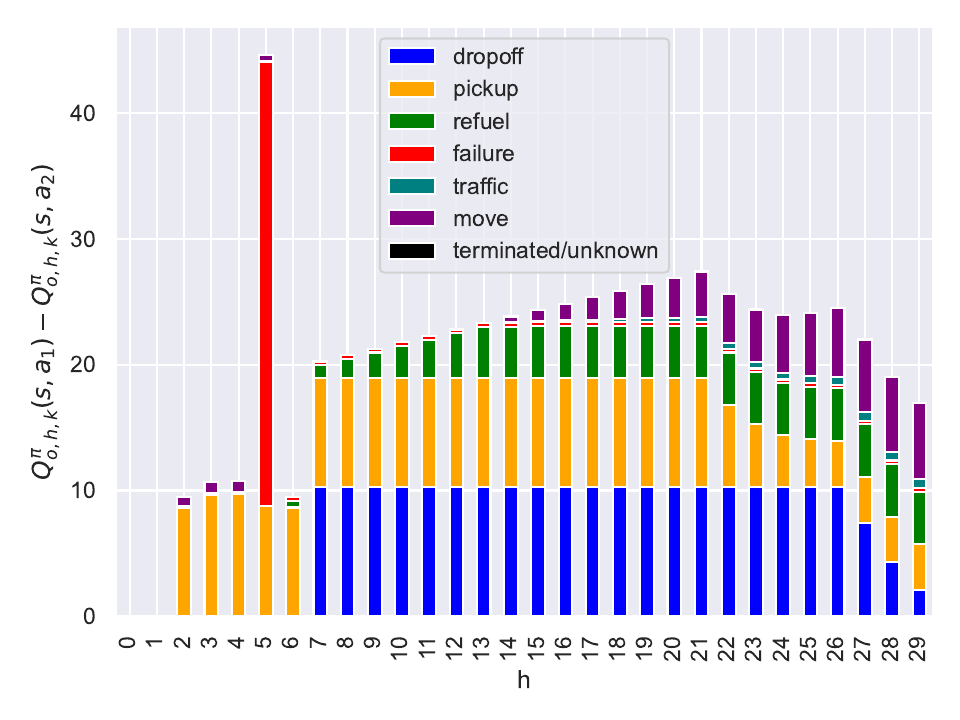}
        \caption{Expected return difference (south - east)}
        \label{fig:explanation-2-contrastive-return-component}
    \end{subfigure}

    \vspace{0.1cm}
    
    \begin{subfigure}{0.48\textwidth}
        \centering
        \includegraphics[trim={15 15 15 10},clip,width=0.9\textwidth]{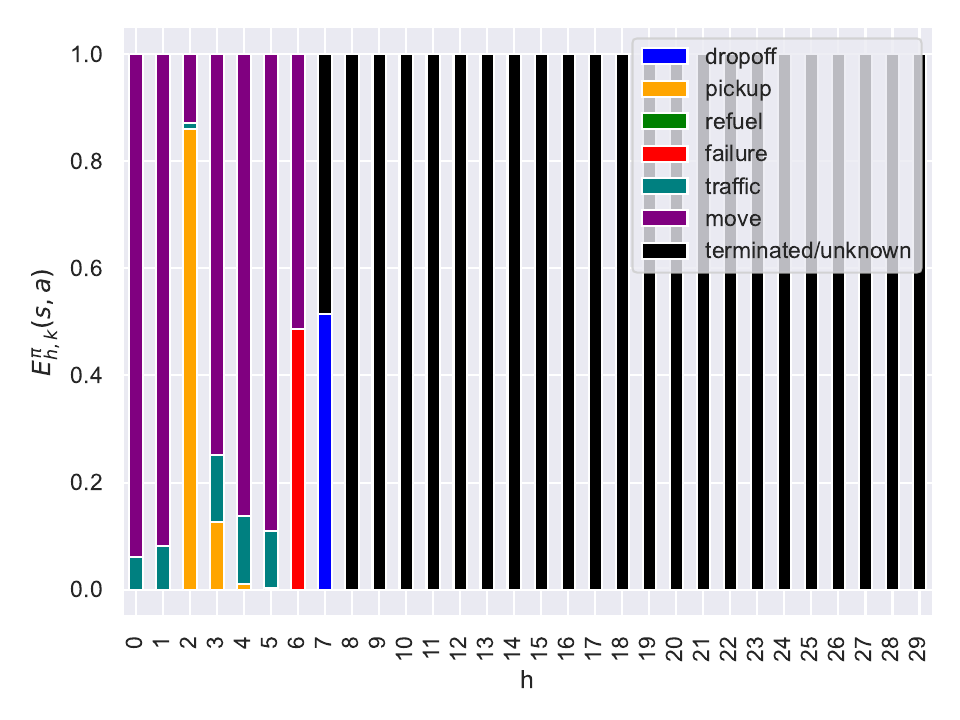}
        \caption{Probabilities of future events (south)}
        \label{fig:explanation-2-pair-1-event}
    \end{subfigure}
    \hfill
    \begin{subfigure}{0.48\textwidth}
        \centering
        \includegraphics[trim={15 15 15 10},clip,width=0.9\textwidth]{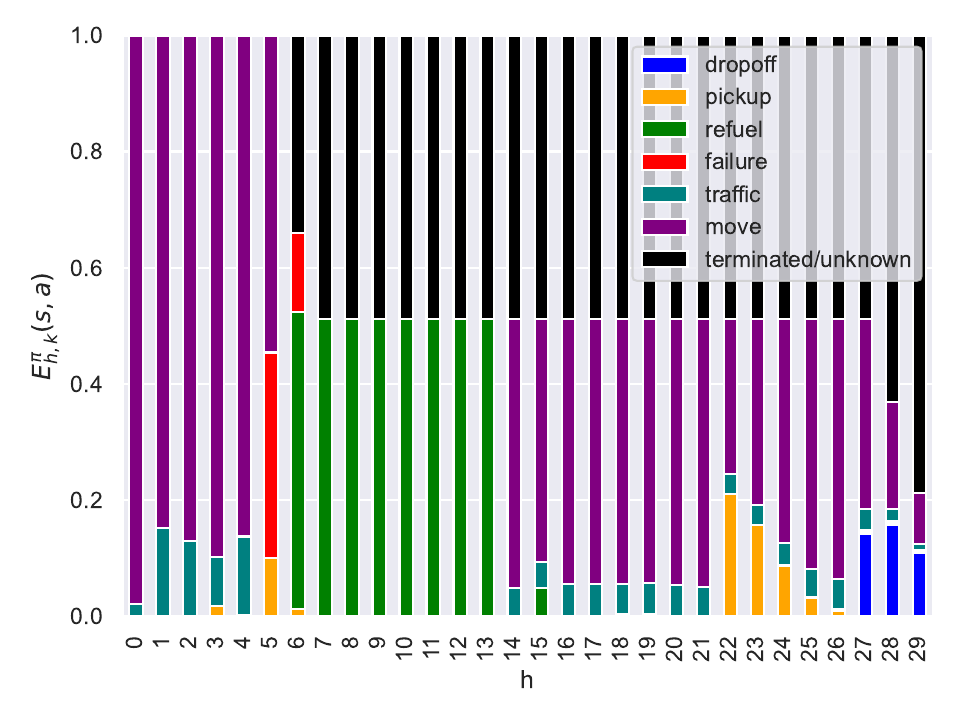}
        \caption{Probabilities of future events (east)}
        \label{fig:explanation-2-pair-2-event}
    \end{subfigure}

    \vspace{0.1cm}
    
    \begin{subfigure}{0.48\textwidth}
        \centering
        \includegraphics[trim={15 15 15 10},clip,width=0.9\textwidth]{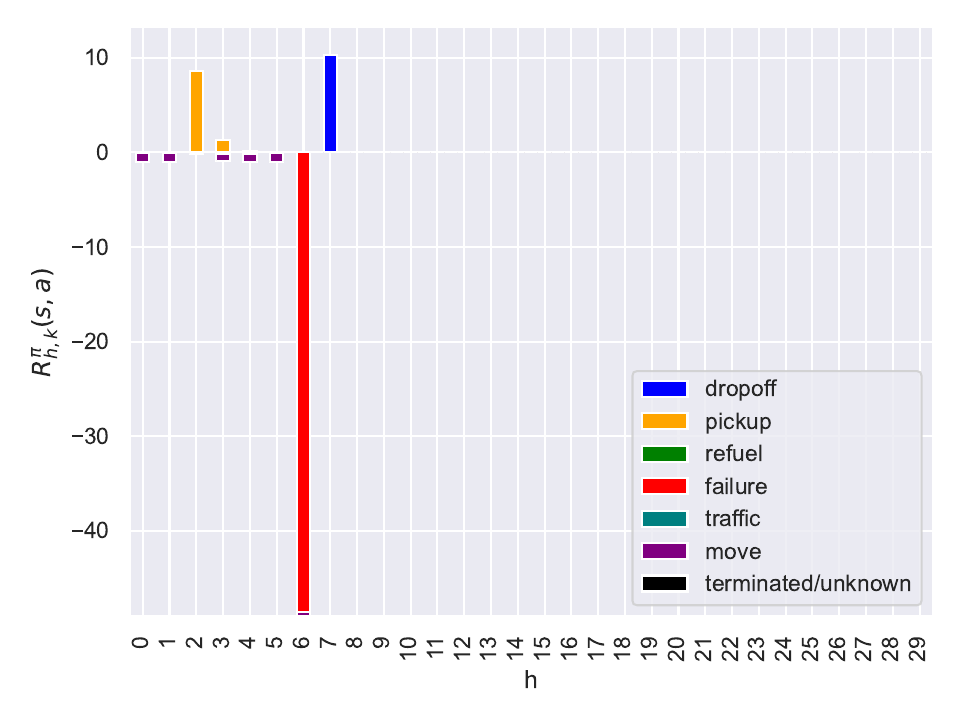}
        \caption{Expected reward components (south)}
        \label{fig:explanation-2-pair-1-reward-component}
    \end{subfigure}
    \hfill
    \begin{subfigure}{0.48\textwidth}
        \centering
        \includegraphics[trim={15 15 15 10},clip,width=0.9\textwidth]{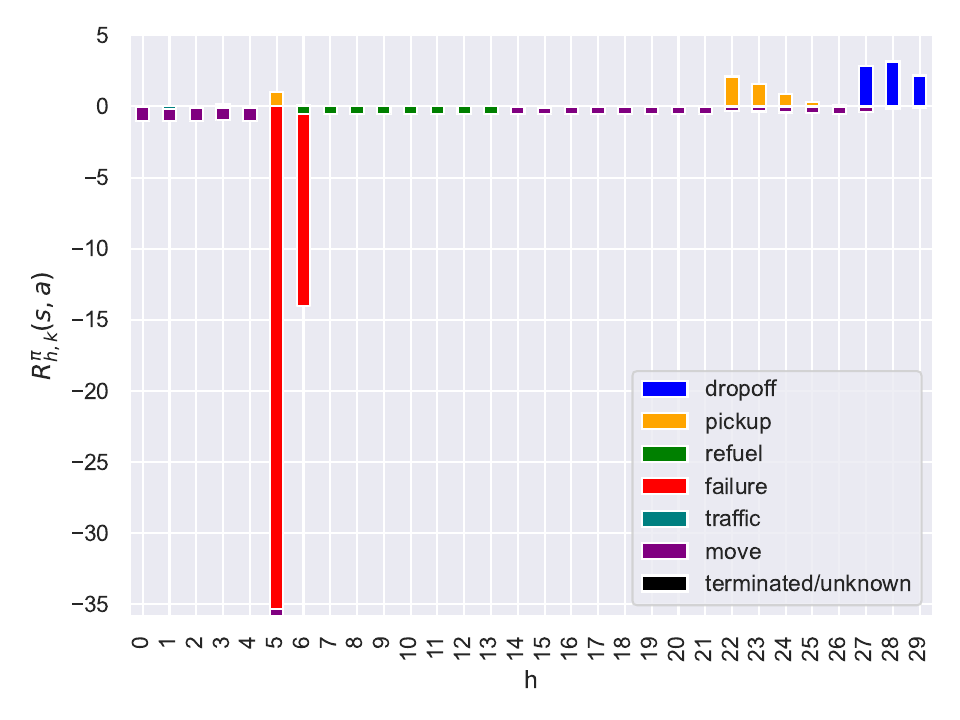}
        \caption{Expected reward components (east)}
        \label{fig:explanation-2-pair-2-reward-component}
    \end{subfigure}

    \caption{Example explanations for an environment state \subref{fig:explanation-2-env} where changing the initial action from moving south (optimal action) to moving east (suboptimal action) significantly impacts the future trajectory. The expected return difference \subref{fig:explanation-2-contrastive-return-component} highlights why moving south is preferable to moving east. The probabilities of future events \subref{fig:explanation-2-pair-1-event}-\subref{fig:explanation-2-pair-2-event} and the corresponding expected future rewards \subref{fig:explanation-2-pair-1-reward-component}-\subref{fig:explanation-2-pair-2-reward-component} illustrate the resulting shift in strategy.}
    \label{fig:explanation-2}
\end{figure}

Another example of explanation is provided in \cref{fig:explanation-2} and examines a state where the choice of the first action significantly affects the strategy and outcome. Here, we compare moving south (optimal action) to moving east (suboptimal action). As before, the probabilities of future events (\cref{fig:explanation-2-pair-1-event} and \cref{fig:explanation-2-pair-2-event}) reveal the taxi’s strategy in each scenario. Moving south initiates a sequence where pickup and dropoff occur quickly (most likely at steps $2$ and $7$), without refueling. However, there remains a relatively high probability (about $0.5$) of running out of fuel at step $6$ due to traffic. Moving east, by contrast, leads the policy to refuel first, as the fuel would be insufficient to complete the task.

To understand why moving south is preferred by the policy, it is helpful to examine the expected future rewards (\cref{fig:explanation-2-pair-1-reward-component} and \cref{fig:explanation-2-pair-2-reward-component}) and their differences (\cref{fig:explanation-2-contrastive-return-component}). The explanations reveal that, starting with moving south, the taxi collects positive reward for pickup and dropoff earlier than in the counterfactual scenario. Both paths involve significant penalties due to potential fuel depletion, but the timing and order of events differ. The expected return difference (\cref{fig:explanation-2-contrastive-return-component}) clarifies the comparison: moving south results in consistently better outcomes starting from step $2$. While the penalty associated with running out of fuel initially favors moving south at step $5$, this advantage disappears by the next step, indicating comparable risks in both paths. However, the costs related to refueling and movement strongly favor moving south, as the eastward trajectory is longer and requires refueling. By step $29$, the difference in expected return is decisively positive for moving south, confirming that it outperforms moving east.

As a note on the terminated event, we observed empirically that prediction errors are primarily underestimations, with no noticeable overestimation. This behavior can be attributed to the initialization of the \gls{fhgvf} values to zero, in conjunction with a cascade effect introduced by the FHTD learning updates across different horizons. When the explainers underestimate the probabilities of certain events, the probability of the terminated event (computed by exclusion) is overestimated. Therefore, we labeled the terminated event as \textit{terminated/unknown} in the figures.

\subsection{Learning Performance}

\begin{table}[t]
    \caption{Errors between \glspl{efo} predicted with \gls{tpd} and ground truth computed with dynamic programming, measured with \gls{mse} and infinity norm $\|\cdot\|_{\infty}$ (i.e., maximum error). The prefixes $\pi$ and $\overline{\pi}$ indicate evaluation on the actions selected and not selected by the policy, respectively, for states encountered in $10^4$ episodes. Means and standard deviations are calculated over $10$ independent runs.}
    \label{tab:errors}
    \centering
    \begin{tabular}{l|c|c|c|c}
        Outcome & $\pi$-\glsentryshort{mse} & $\overline{\pi}$-\glsentryshort{mse} & $\pi$-$\|\cdot\|_{\infty}$ & $\overline{\pi}$-$\|\cdot\|_{\infty}$ \\
        \hline\hline
        Dropoff & $(1.86 \pm 0.19) 10^{-4}$ & $(1.10 \pm 0.03) 10^{-4}$ & $(1.20 \pm 0.22) 10^{-1}$ & $(4.40 \pm 1.21) 10^{-1}$  \\
        \hline
        Pickup & $(1.03 \pm 0.16) 10^{-4}$ & $(9.17 \pm 0.79) 10^{-5}$ & $(1.78 \pm 0.21) 10^{-1}$ & $(2.09 \pm 0.16) 10^{-1}$ \\
        \hline
        Refuel & $(4.48 \pm 0.27) 10^{-5}$ & $(1.60 \pm 0.15) 10^{-4}$ & $(2.28 \pm 0.32) 10^{-1}$ & $(2.05 \pm 0.24) 10^{-1}$ \\
        \hline
        Failure & $(3.75 \pm 0.64) 10^{-6}$ & $(7.39 \pm 6.56) 10^{-6}$ & $(1.68 \pm 0.27) 10^{-1}$ & $(3.36 \pm 0.79) 10^{-1}$ \\
        \hline
        Traffic & $(2.46 \pm 0.07) 10^{-4}$ & $(2.42 \pm 0.15) 10^{-4}$ & $(2.80 \pm 0.20) 10^{-1}$ & $(4.77 \pm 1.62) 10^{-1}$ \\
        \hline
        Move & $(4.91 \pm 0.30) 10^{-4}$ & $(5.86 \pm 0.43) 10^{-4}$ & $(2.80 \pm 0.20) 10^{-1}$ & $(7.77 \pm 0.66) 10^{-1}$ \\
    \end{tabular}
\end{table}

Accurate predictions are crucial for reliable explanations. To quantitatively assess the quality of the explanations, we first computed the exact \glspl{fhgvf} using dynamic programming and derived the corresponding ground-truth \glspl{efo}. We then compared the approximate \glspl{efo} predicted by our method against the ground truth.

For the evaluation, we collected states by running the policy for $10^4$ episodes. For each state encountered, we computed the approximate and exact \glspl{efo} for all possible actions, including both the actions selected by the policy and the other actions. The evaluation metrics are reported separately for on-policy and off-policy actions and denoted by $\pi$ and $\overline{\pi}$, respectively. We repeated the evaluation process for $10$ independent runs. The results, presented in \cref{tab:errors}, demonstrate that the explainers accurately predict the \glspl{efo} for all the actions. This finding suggests that the qualitative explanations generated by our method are reliable and trustworthy.

It is important to note that dynamic programming has two key limitations that hinder its practical adoption and motivate our data-driven method. First, it assumes known transition probabilities, which contradicts the core strength of \gls{rl}---the ability to work with unknown dynamics. Second, it does not generalize or scale well to continuous or high-dimensional \glspl{mdp}. As a result, dynamic programming is impractical outside of controlled, low-dimensional simulated \glspl{mdp}.

\section{Conclusion}
\label{sec:conclusion}

In this paper, we introduced \glsentryfull{tpd}, an \gls{xrl} method designed to clarify \gls{rl} decision-making by explaining actions in terms of their \glsentryfullpl{efo} across time steps. We showed how \glspl{gvf}, including the value functions used for decision-making in \gls{rl} control algorithms, can be decomposed along the temporal dimension to provide insights into when specific outcomes are expected to occur. To achieve this, we leveraged \glsentrylong{fhtd} learning to devise an off-policy approach that efficiently exploits the sequential nature of \glspl{mdp}. Our results demonstrated that \gls{tpd} delivers both accurate and interpretable explanations that clarify the policy strategy for a prediction horizon of interest, building trust among human users. Additionally, we discussed how contrastive explanations in terms of \glspl{efo} can also serve as a tool to guide fine-tuning of the reward function for better alignment with human expectations.

There are several promising directions for future research. First, experimenting with this method in continuous problems, such as radio network optimization, could validate its performance in more realistic and complex settings. Second, integrating uncertainty estimates into predictions---using techniques like conformal prediction to compute confidence intervals---could be beneficial to inform users about the reliability of \gls{efo} predictions. Lastly, investigating approaches to predict the full distribution of future outcomes, rather than focusing solely on expected values, could be particularly beneficial in problems where rewards cannot be easily decomposed into discrete events.

\acks{%
This work was partially supported by the \gls{ssf}.
}

\appendix
\section{Convergence of \glsentryshort{fhtd} Learning}
\label{apx:theorem}

In this appendix, we prove \cref{thm:convergence} presented in \cref{sec:learning-efos}.

\begin{repeattheorem}{thm:convergence}
    Assume that
    \begin{enumerate*}[label=(\roman*)]
     \item the behavioral policies  $(\beta_t)_t, \beta_t:{\cal S}\to \Delta({\cal A})$, with $a_t \sim \beta_t(\cdot|s_t)$, ensure that every state-action pair $(s,a)$ is visited infinitely often and
     \item for every $h$ the  sequences $\alpha_{t,h}(s,a)$ are positive, non-increasing in $t$, satisfying $\sum_{n=1}^\infty \alpha_{n,h}(s,a)=\infty$, $\sum_{n=1}^\infty \alpha_{n,h}(s,a)^2<\infty$.
    \end{enumerate*}
    Then, under these assumptions and using the update in \cref{eq:fhtd-learning}, we have that $\lim_{t\to\infty} \hat{Q}_h^{(t+1)}(s,a)= Q_h^\pi(s,a)$ almost surely for every $(s,a)$.
\end{repeattheorem}
\begin{proof}
Denote by $\hat{Q}_h^{(t)}(s,a)$ the fixed-horizon $Q$-value for horizon $h$ at time step $t$ in a given state-action pair $(s,a)$.  We can rewrite the update rule in \cref{eq:fhtd-learning} as:
\[
    \hat{Q}_h^{(t+1)}(s, a) = 
    \begin{cases}
        \hat{Q}_h^{(t)}(s, a) + \alpha_{t,h}(s,a) \left( r_t+ \gamma \hat{Q}_{h-1}^{(t)}(s_{t+1},a_{t+1}) - \hat{Q}_h^{(t)}(s, a) \right), & \text{if }(s,a)=(s_t,a_t) \\
        \hat{Q}_h^{(t)}(s, a), & \text{otherwise}
    \end{cases}
\]
where $a_{t+1}\sim \pi(\cdot|s_{t+1}), r_t\sim q(s_t,a_t)$, with $r(s,a)=\mathbb{E}_{r\sim q(s,a)}[r]$ and $Q_0^t(s,a)=0$ for every $t$ and $(s,a)$.

We exploit the structure of the update scheme, that is $Q_{h-1}^\pi$ does not depend on $Q_h^\pi,Q_{h+1}^\pi,\dots$ to prove the claim by induction.

\paragraph{Proof of convergence for $h=1$.} We first prove that the iteration converges for $h=1$, which is given by
\[
\hat{Q}_h^{(t+1)}(s, a) = \begin{cases}
\hat{Q}_h^{(t)}(s, a) + \alpha_{t,h} \left( r_t- \hat{Q}_h^{(t)}(s, a) \right), & \text{if }(s,a)=(s_t,a_t) \\
\hat{Q}_h^{(t)}(s, a), & \text{otherwise}
\end{cases}
\]
The proof is rather standard. Let $\{ t_n \} $ be the sequence of time steps at which the state-action pair $ (s_t, a_t) = (s, a)$ occurs. We define $ n $ as the index for these occurrences. The update rule at these times becomes:
\[
\hat{Q}_1^{(n+1)}(s, a) = \hat{Q}_1^{(n)}(s, a) + \alpha_{n,1}(s, a) \left( r_{t_n} - \hat{Q}_1^{(n)}(s, a) \right),
\]
where  $\alpha_{n,1}(s, a) = \alpha_{t_n,1}(s,a)$. Now, define the error term $ M_{n+1}= r_{t_n} - r(s, a)$ and note that $\mathbb{E}[M_{n+1}]=0$. Hence, the update equation can now be written in the form of a standard stochastic approximation:
\[
\hat{Q}_1^{(n+1)}(s, a) = \hat{Q}_1^{(n)}(s, a) + \alpha_{n,1}(s, a) \left( g(\hat{Q}_1^{(n)}(s, a)) + M_{n+1} \right),
\]
where $g(x) = r(s, a) - x$. Clearly $g$ is Lipschitz, of constant $1$. Since every state-action pair $(s,a)$ is visited infinitely often, the rewards are  bounded, and the sequence $\{ \alpha_{n,1}(s, a) \}$ is positive, non-increasing, and satisfy $
\sum_{n=1}^{\infty} \alpha_{n,1}(s, a)  = \infty, \quad \sum_{n=1}^{\infty} [\alpha_{n,1}(s, a) ]^2 < \infty
$ we have that the sequence $\hat{Q}_1^{(n+1)}$ converges almost surely $\lim_{n\to\infty} \hat{Q}_1^{(n)}(s,a)=r(s,a)$ by standard stochastic approximation techniques \citep{borkar-2023}.

\paragraph{Proof of convergence for generic $h$.} Assuming the iteration converges for $h-1$, we now consider the iteration for the $h$-th fixed horizon $Q$-value. As for $h=1$, we denote  by $\{ t_n \} $  the sequence of time steps at which the state-action pair $ (s_t, a_t) = (s, a)$ occurs. We define $ n $ as the index for these occurrences. The update rule at these times becomes:
\[
\hat{Q}_h^{(n+1)}(s, a) = \hat{Q}_h^{(n)}(s, a) + \alpha_{n,h}(s, a)  \left( r_{t_n} +\hat{Q}_{h-1}^{(n)}(s_{t+1},a_{t+1})- \hat{Q}_h^{(n)}(s, a) \right),
\]
where  $\alpha_{n,h}(s, a) = \alpha_{t_n,h}(s,a)$.  Now, let:
\begin{align*}
M_{n+1} &= r_{t_n} + Q_{h-1}^\pi(s_{t+1},a_{t+1}) - r(s,a) - \mathbb{E}[Q_{h-1}^\pi(s^{\prime},a^{\prime})|s=s_t,a=a_t,\pi],\\
g(x)&=r(s,a) + \mathbb{E}[Q_{h-1}^\pi(s^{\prime},a^{\prime})|s,a,\pi]-x,\\
E_{n+1}&=\hat{Q}_{h-1}^{(n)}(s_{t+1},a_{t+1})-Q_{h-1}^\pi(s_{t+1},a_{t+1}).
\end{align*}
Hence, the update rule in $(s_t,a_t)=(s,a)$ becomes:
\begin{align*}
\hat{Q}_h^{(n+1)}(s, a) &= \hat{Q}_h^{(n)}(s, a) + \alpha_{n,h}(s, a)  \left(g\left(\hat{Q}_h^{(n)}(s,a)\right)
+ M_{n+1}+E_{n+1}\right).
\end{align*}
Now, observe that: (1) $M_{n+1}$ is a martingale difference sequence; (2) $\hat{Q}_h^{(t)}$ remains bounded; (3) $g(x)$ is Lipschitz continuous; (4) $E_{n+1}=o(1)$ (bounded), from the induction hypothesis (hence $E_{n+1}\to 0$ almost surely). Then, since every state-action pair $(s,a)$ is visited infinitely often, and the sequence $\{ \alpha_{n,h}(s, a) \}$ is positive, non-increasing, and satisfy $\sum_{n=1}^{\infty} \alpha_{n,h}(s, a)  = \infty, \quad \sum_{n=1}^{\infty} [\alpha_{n,h}(s, a) ]^2 < \infty$ we have that the sequence $\hat{Q}_h^{(n)}$ converges almost surely $\lim_{n\to\infty} Q_h^n(s,a)=r(s,a) + \mathbb{E}[Q_{h-1}^\pi(s^{\prime},a^{\prime})|s,a,\pi]$ by standard stochastic approximation techniques \citep{kushner-1978,borkar-2023}. Regarding the condition $E_{n+1}=o(1)$, see \citet[Theorem 1]{wang-1996},\citet[Theorem 2.3.1]{kushner-1978}  or \citet[Chapter 2.2]{borkar-2023}.
\end{proof}

\section{Training Details}
\label{apx:training-details}

\begin{table}[t]
    \caption{Hyperparameters used for training the Q-learning policy.}
    \label{tab:hyperparams-q-learning}
    \centering
    \begin{tabular}{l|c}
        Hyperparameter & Value \\
        \hline\hline
        Total timesteps & \num{5e5} \\
        \hline
        Discount factor & \num{0.99} \\
        \hline
        Learning rate & 0.1 \\
        \hline
        Start exploration rate & \num{1.0} \\
        \hline
        Final exploration rate & \num{0.05} \\
        \hline
        Timesteps of exploration linear decay & \num{2.5e5} \\
    \end{tabular}
\end{table}

In the adapted Taxi environment, we trained a policy using Q-learning with action masking. The hyperparameters are detailed in \cref{tab:hyperparams-q-learning}. Action masking limited the agent’s choices to only valid actions, significantly boosting sample efficiency due to the high number of invalid actions in this environment. Q-values for all state-action pairs were initialized to zero.

Based on the trained policy, we developed an explainer for each of the following events: \textit{dropoff}, \textit{pickup}, \textit{refuel}, \textit{failure}, \textit{traffic}, and \textit{move}. The same set of hyperparameters, listed in \cref{tab:hyperparams-tpd}, was used across all explainers. We initialized the \gls{fhgvf} values to zero for all state-action pairs.

\begin{table}[t]
    \caption{Hyperparameters used for training the explainers.}
    \label{tab:hyperparams-tpd}
    \centering
    \begin{tabular}{l|c}
        Hyperparameter & Value \\
        \hline\hline
        Horizon & \num{30} \\
        \hline
        Total timesteps & \num{5e6} \\
        \hline
        Discount factor & \num{1.0} \\
        \hline
        Learning rate & 0.1 \\
        \hline
        Exploration rate & 0.2 \\
    \end{tabular}
\end{table}

\vskip 0.2in
\bibliography{bibliography}

\end{document}